\documentclass[letterpaper, 10 pt, conference]{ieeeconf}  

\IEEEoverridecommandlockouts                              

\usepackage{graphicx} 
\usepackage{times} 
\usepackage{amsmath} 
\usepackage{amssymb}  
\usepackage{cite}
\usepackage{subfigure}
\usepackage[noend]{algorithmic}

\usepackage{bm}

\newtheorem{proposition}{Proposition}
\newtheorem{definition}{Definition}

\title{\LARGE \bf
Planning in Stochastic Environments with Goal Uncertainty
}

\author{Sandhya Saisubramanian$^{1}$ \and Kyle Hollins Wray$^{1,2}$ \and Luis Pineda$^{1}$ \and Shlomo Zilberstein$^{1}$
\thanks{Support for this work was provided in part by the U.S. National Science Foundation grants IIS-1524797 and IIS-1724101.}
\thanks{$^{1}$College of Information and Computer Sciences, University of Massachusetts Amherst, MA, USA.}%
\thanks{$^{2}$Alliance Innovation Lab Silicon Valley, Santa Clara, CA, USA. }%
}
\begin{document}

\maketitle
\thispagestyle{empty}
\pagestyle{empty}

\begin{abstract}

We present the Goal Uncertain Stochastic Shortest Path (GUSSP) problem | a general framework to model path planning and decision making in stochastic environments with goal uncertainty. The framework extends the stochastic shortest path (SSP) model to dynamic environments in which it is impossible to determine the exact goal states ahead of plan execution. GUSSPs introduce flexibility in goal specification by allowing a belief over possible goal configurations. The unique observations at potential goals helps the agent identify the true goal during plan execution. The partial observability is restricted to goals, facilitating the reduction to an SSP with a modified state space. We formally define a GUSSP and discuss its theoretical properties. We then propose an admissible heuristic that reduces the planning time using FLARES | a start-of-the-art probabilistic planner. We also propose a determinization approach for solving this class of problems. Finally, we present empirical results on a search and rescue mobile robot and three other problem domains in simulation.
\end{abstract}

\section{Introduction and Related Work}
Autonomous robots acting in the real world are often faced with tasks that require path planning in stochastic environments. These problems are typically modeled as a Stochastic Shortest Path (SSP) problem, which generalizes both finite and infinite-horizon Markov decision processes (MDPs) and is a convenient framework to model goal-driven problems~\cite{bertsekas1991analysis}. The objective in an SSP is to devise a sequence of actions such that the expected cost of reaching a \emph{known} goal state from the start state is minimized. 

Consider a search and rescue domain (Figure~\ref{fig:sr}), a motivating example where the robot has to devise a cost minimizing path to rescue people from a building~\cite{kitano1999robocup,pineda2015continual}. While the number of victims and the map of the building may be provided to the robot, only potential victim locations may be known ahead of plan execution. The unavailability of the exact goal states (victim locations) during planning time prevents the problem from being modeled as a standard MDP or SSP. In this work we assume that the exact goal states may be hard to identify, but historical data or noisy sensors allow the robot to establish a belief distribution over possible victim locations. The search and rescue domain is an instance of the \emph{optimal search for stationary targets}~\cite{hansen2007indefinite,stone2016search,lanillos2012minimum,trevizan13finding} | a class of problems in which  the target's exact location is unknown to the robot, but the robot can observe its current location and determine whether the target is in the current location. Hence, we assume that the robot is given well-defined goal conditions, but has uncertainty about the states that satisfy these goal conditions.
\begin{figure}
	\centering
	\subfigure[Problem setting]{\includegraphics[width=1.8in, height=.90in]{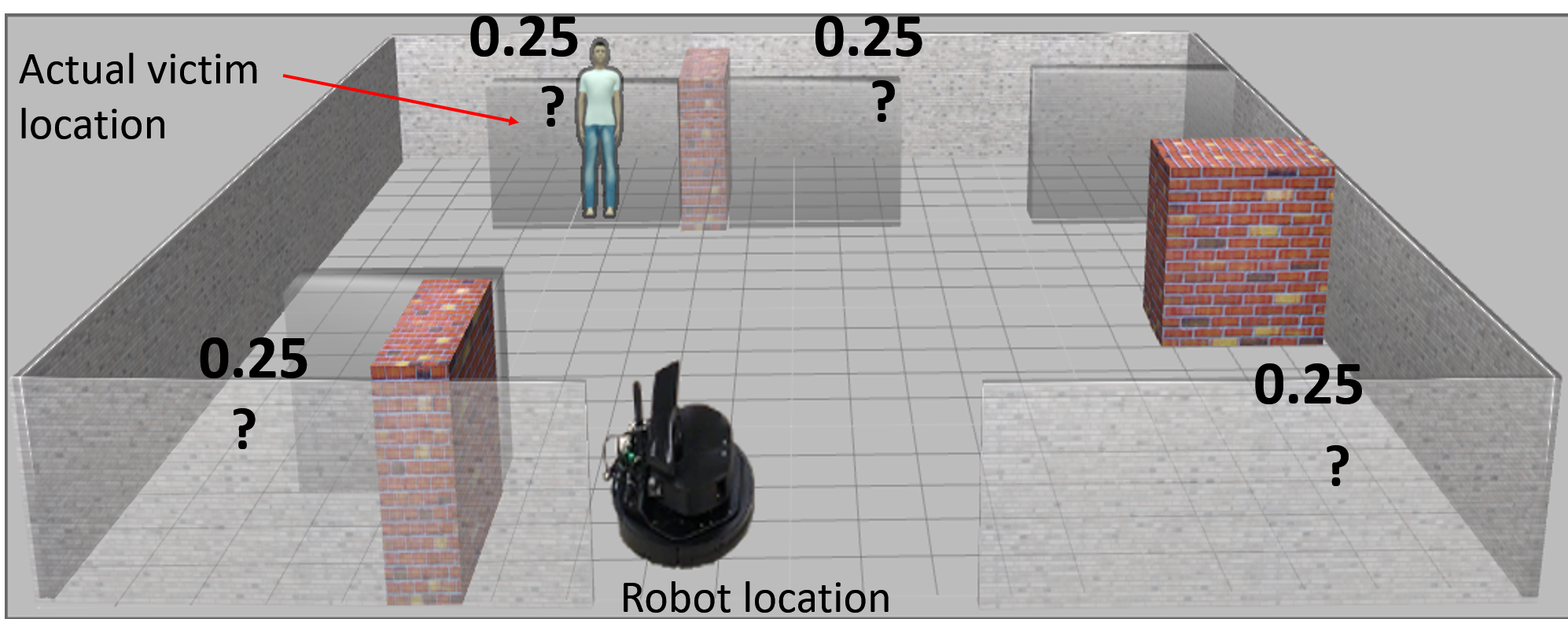}}
	\subfigure[Experimental setting]{\includegraphics[width=1.5in,height=.88in]{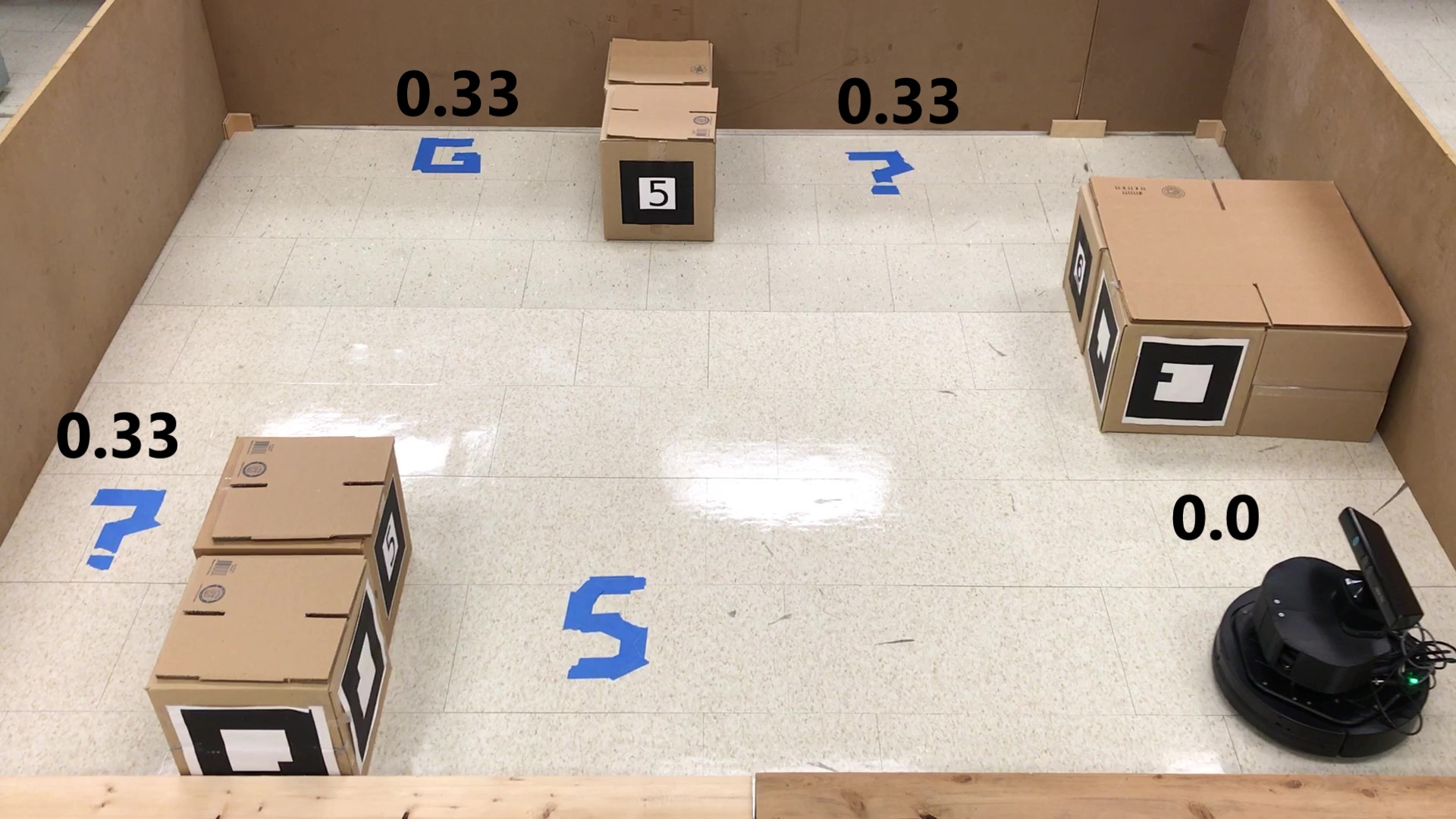}}
	\vspace{-3pt}
	\caption{An example of a search and rescue problem with goal uncertainty, showing 
	the initial belief (left) and the corresponding experimental setting 
	with a mobile robot and updated beliefs (right). The question marks indicate potential victim locations and values denote the robot's belief. S denotes the robot's start location and G is the actual victim location (goal). The robot updates its belief about the victim locations based on its observations.}
	\label{fig:sr}
	\vspace{-15pt}
\end{figure}

In the existing literature~\cite{nie2016searching,ong2010planning}, such problems are typically modeled as a Partially Observable MDP (POMDP)~\cite{kaelbling1998planning}, a rich framework that facilitates modeling various forms of partial observability. However, POMDPs are much harder to solve~\cite{papadimitriou1987complexity}. The partially observable SSPs (POSSPs) extend the SSP framework to settings with partially observable states, offering a class of indefinite-horizon, undiscounted POMDPs that rely on state-based termination~\cite{patek2001partially}. Other relevant POMDP variants are the Mixed Observable MDPs~\cite{ong2010planning} that model problems with both fully observable and partially observable state factors and the Goal POMDPs~\cite{bonet2009solving} that are goal-based with no discounting. These models are solved using POMDP solvers and are difficult to solve optimally. They also suffer from limited scalability due to their computational complexity~\cite{papadimitriou1987complexity}. Our objective in this work is to efficiently solve problems with goal uncertainty by leveraging the fully observable components of the problem.

We present goal uncertain SSP (GUSSP), a framework specifically designed to model problems with imperfect goal information by allowing for a probabilistic distribution over possible goals. GUSSPs fit well with many real-world settings where it is easier and more realistic to have belief over goal configurations, rather than exact knowledge about the goal states. The observation function in a GUSSP facilitates the reduction to an SSP, enabling the computation of tractable and optimal solutions. We address settings where the goals do not change over time and we assume the existence of a unique observation that allows the robot to accurately identify a goal when it reaches one.  

Our key contributions are: (i) a formal definition of GUSSP and its theoretical properties; (ii) a domain-independent, admissible heuristic that can accelerate probabilistic planners; (iii) a determinization approach for solving GUSSPs; and (iv) empirical evaluation on three realistic domains in simulation and on a mobile robot.

\section{Background: Stochastic Shortest Path}
A \textbf{Stochastic Shortest Path} (SSP) is a more general formulation of an MDP to model goal-oriented problems that require sequential decision making under uncertainty. Formally, an SSP is defined by the tuple $\langle S,A,T, C,s_0,S_G \rangle$, where $S$ is a finite set of states; $A$ is a finite set of actions; $T:\!S \!\times A\times S \rightarrow[0,1]$ is the transition function representing the probability of reaching a state $s'\in S$ by executing an action $a \in A$ in state $s \in S$, and denoted by $T(s,a,s')$; $C:\!S\times A \rightarrow \mathbb{R^+}\cup\{0 \}$ is the cost function representing the cost of executing action $a \in A$ in state $s \in S$, and denoted by $C(s,a)$; $s_0 \in S$  is the initial state; and $S_G \subseteq S$ is the set of absorbing goal states. The cost of an action is positive in all states except absorbing goal states, where it is zero. An SSP is an MDP with no discounting, that is, the discount factor $\gamma\!=\!1$. The objective in an SSP is to minimize the expected cost of reaching a goal state from the start state. It is assumed that there exists at least one \emph{proper policy}, one that reaches a goal state from any state $s$ with probability 1. The optimal policy, $\pi ^*$, can be extracted using the value function defined over the states, $V^*(s)$: 
\begin{align} 
V^*(s) &= \min_a ~~Q^*(s,a), \hspace{10pt} \forall s \in S \nonumber \\ 
Q^*(s,a)& =  C(s,a) \!+\! \sum_{s'}T(s,a,s') V^*(s'), \forall (s, a) \nonumber
\end{align} 
with $Q^*(s,a)$ denoting the optimal Q-value of the action $a$ in state $s$ in the SSP. While SSPs can be solved in polynomial time in the number of states, many problems of interest have a state-space whose size is exponential in the number of variables describing the problem~\cite{littman1997probabilistic}. This complexity has led to the use of various approximate methods that either ignore stochasticity or use a short-sighted labeling approach for quickly solving the problem.

\section{Goal Uncertain Stochastic Shortest Path}
A goal uncertain stochastic shortest path (GUSSP) problem is a generalized framework to model problems with goal uncertainty. A GUSSP is an SSP in which the agent may not know initially the exact set of goal states ($S_G$, which does not change over time), and instead can obtain information about the goals via observations. 
\begin{definition}
	A \textbf{goal uncertain stochastic shortest path problem} is a tuple $\langle X,S,A,T,C,s_0,S_G,P_G,\Omega, O \rangle$ where 
	\begin{itemize}
		\setlength\itemsep{.01 em}
		\item $S,A,T,C,s_0,S_G$ denote an underlying SSP with discrete states and actions, and $S_G$ unknown to the agent;
		\item $P_G \subseteq S$ is the set of potential goals such that $S_G \subseteq P_G$; 
		\item $X\!=\!S\!\times\!G$ is the set of states in the GUSSP with $G\!=\!2^{P_G}\!\setminus\!\{\emptyset\}$ denoting the set of possible goal configurations;  
		\item $\Omega$ is a finite set of observations corresponding to the goal configurations, $\Omega = G$; and
		\item $O\!:\!A\times X\times \Omega \rightarrow [0,1]$ is the observation function denoting the probability of receiving an observation, $\omega\!\in\!\Omega$, given action $a\in A$ led to state $x'$ with probability $O(a,x',\omega) \equiv Pr(\omega \vert a,x')$.
	\end{itemize}
\end{definition}
Each state is represented by $x\!=\!\langle s,g \rangle$, with $s\!\in\!S$ and $ g\in G$. GUSSPs have mixed observable state components as $s$ is fully observable. Each $g\!\in\!G$ represents a goal configuration (set of states), thus permitting multiple true goals in the model, $\vert S_G \vert\!\geq\!1$. Every action in each state produces an observation, $\omega\!\in\!\Omega$, which is a goal configuration and thus provides information about the true goals. The agent's belief about its current state is denoted by $b(x)$, with $x\!=\!\langle s,g\rangle$; that is, the belief about $g\!=\!S_G$. The initial belief is denoted by $b_0\langle s_0,g \rangle\!\in\![0,1], \forall g\in G$, where $s_0$ is the start state of the SSP. The process terminates when the agent reaches a state $x$ with $b(x)\!=\!1$ and $s\!\in\!g$. SSPs are therefore a special type of GUSSPs with a collapsed initial belief over the goals.

As in (PO)SSP, we assume in a GUSSP: (1) the existence of a proper policy with finite cost, (2) all improper policies have infinite cost, and (3) termination is perfectly recognized. 

\vspace{4pt}
\noindent{\textbf{Observation Function~}}
In a GUSSP, an observation function is characterized by two properties. First, to perfectly recognize termination, all potential goals are characterized by a unique belief-collapsing (when the belief over a state is either 1 or 0) observation. That is, at potential goal states, if $s'\in g'$, then $\forall a \in A$:
\begin{align}
O(a,x',\omega)   = \begin{cases}
1\quad \text{if } g' = \omega   \\
0 \quad \text{otherwise.}
\end{cases}
\label{uni-obs}
\end{align}
Second, the observation function is \emph{myopic}, providing information only about the current state or the potential goals in the immediate vicinity. This is based on real-world settings with limited range sensors and the exploration and navigation approaches for robots that acknowledge the perceptual limitations of robots~\cite{biswas2013multi}. The \emph{landmark states}, $L_s$, provide accurate information about certain potential goals. Each $s\!\in\!L_s$ provides observations about a subset of potential goals in the vicinity with $\Omega_{s}$ denoting the corresponding set of observations. Each $\omega\!\in\!\Omega_{s}$ provides information about the maximal set of potential goals in the vicinity. Other non-potential goal states, $s' \notin L_s$, provide no information about the true goals. Therefore, the observation at non-potential goal states is $\forall a\in A$:
\begin{align}
O(a,x',\omega)  \!=\!\begin{cases}
1 \hspace{5pt} \text{if } s' \in L_s\!\land \omega \subseteq g'  \land \omega\!\in\!\Omega_{s'}\\
0 \hspace{5pt} \text{if } s' \in L_s\!\land \omega \not\subseteq g'\land \omega\!\in\! \Omega_{s'} \\
\frac{1}{\vert\Omega \vert} \hspace{5pt} \text{if } s' \notin L_s \\
\end{cases}, \nonumber
\end{align}
with $x\!=\!\langle s,g\rangle$ and $x'\!=\!\langle s',g'\rangle$. 
The potential goals along with the landmark states are called \emph{informative states}, $\mathcal{I}\!=\!P_G \cup L_s$, since they provide information about the true goals through deterministic observations. Thus, our observation function satisfies the minimum information required for state-based termination. In the next section, we discuss a more general setting where every state may have a noisy observation regarding the true goals.

\vspace{4pt}
\noindent{\textbf{Belief Update~}} A belief $b$ is a probability distribution over $X$, $b(x)\!\in\! [0,1],\forall x\in X$ and $\sum_{x\in X} b(x)\!=\!1$. The set of all reachable beliefs forms the belief space $B\subseteq \Delta^n $, where $\Delta^n$ is the standard $(n\!-\!1)$-simplex. The agent updates the belief $b'\!\in\!B$, given the action $a \in A$, an observation $\omega \in \Omega$, and the current belief $b$. Using the multiplication rule, the updated belief for $x'\!=\!\langle s',g'\rangle$ is:
\begin{align}
b'(x'\vert b,a,\omega) &=  Pr(g'\vert b,a,\omega, s')\, Pr(s'\vert b,a,\omega,s) \nonumber \\
&= Pr(g'\vert b,a,\omega, s')\,T(s,a,s') \nonumber \\
Pr(g'\vert b,a,\omega,s') &= \eta Pr(\omega\vert b,a,s',g') Pr(g'\vert b,a,s') \nonumber \\
&= \eta O(a,x',\omega)\sum_{g\in G}Pr(g',g\vert b,a,s')\nonumber \\
&= \eta O(a,x',\omega) b(g), \label{g-update}
\end{align}
where $\eta = Pr(\omega \vert b,a,s')^{-1}$ is a normalization constant and $b(g)$ is the belief over the goal configuration. Therefore,
\begin{align}
b'(x'\vert b,a,\omega)= \eta O(a,x',\omega) b(g) T(s,a,s'). \label{bel-gussp}
\end{align}

\vspace{4pt}
\noindent{\textbf{Policy and Value~}} 
The agent's objective in a GUSSP is to minimize the expected cost of reaching a goal, $ \min_{\pi \in \Pi} \mathbb{E}\Big[\sum_{t=0}^{h} C(x_t,a_t)\Big| \pi \Big]$, 
where $x_t$ and $a_t$ denote the agent's state and action at time $t$ respectively, and $h \in \mathbb{N}$ denotes the horizon. A policy $\pi\!:\!B\!\rightarrow A$ is a mapping from belief $b\in B$ to an action $a \in A$. The value function for a belief, $V\!:\!B\rightarrow\!\mathbb{R}$ is the expected cost for a fixed policy $\pi$ and a horizon $h$. The Bellman optimality equation for GUSSPs follows from POMDPs:
\[V(b) = \min_{a \in A}\big[C(b,a)+ \sum_{\omega\in \Omega}Pr(\omega | b,a)V(b'_{a\omega})  \big],\]
where $b'_{a\omega}$ is the updated belief following Equation~(\ref{bel-gussp}), $C(b,a)\!=\!\sum_{x} b(x)C(x,a)$, $x\!=\!\langle s,g\rangle$, and $x\!=\!\langle s',g'\rangle$. A proper policy, $\pi$, in a GUSSP guarantees termination in a finite expected number of steps, $V^\pi(b_0)<\infty$.

\section{Theoretical Analysis}
In a GUSSP, the observation function critically affects the number of reachable beliefs. We begin with analyzing how the number of beliefs may grow in the more general (non-myopic observation) setting and then show that a GUSSP with myopic observations has finite reachable beliefs.

In a GUSSP with non-myopic observations, the nonpotential goal states, $\forall s \notin L_s$, provide stochastic observations about the true goals, resulting in infinitely many reachable beliefs. While this is a trivial fact, it is useful to understand the growth in complexity of the problem and it provides an important link to POMDPs via the belief MDP. The following proposition formally proves this complexity.

\begin{proposition}
	For all horizon $h\!>\!0$, the belief-MDP of a GUSSP with non-myopic observations may have $\mathcal{O}(\vert \Omega\vert ^h)$ states.\label{prop:exp}
\end{proposition}
\begin{proof}
	By construction, we map GUSSP with non-myopic observations to a belief MDP $\langle B, \mathcal{A}, \mathcal{\tau}, \mathcal{\rho} \rangle$ with a horizon $h$~\cite{kaelbling1998planning}. Let $ \mathcal{R}(b_0)$ denote the set of reachable beliefs in the GUSSP. The set of states in the MDP is the set of reachable beliefs from $b_0$ in the GUSSP, $B = \mathcal{R}(b_0)$. The set of actions in the GUSSP are retained in the MDP, $\mathcal{A}\!=\!A$. The cost function $\rho(b,a)\!=\!\sum_{x\in X}b(x)C(x,a)$, where $C(x,a)$ corresponds to cost function of GUSSP. The transition function for the belief MDP is the probability of executing action $a\in \mathcal{A}$ in belief state $b \in B$ and reaching the reaching belief $b'$, and denoted by $\mathcal{\tau} (b,a,b')$, is:
	\begin{align}
	\mathcal{\tau} (b,a,b') &= \sum_{\omega \in \Omega }Pr(b',\omega\vert b,a) \nonumber \\
	&= 	\sum_{\omega \in \Omega }Pr(b'\vert b,a, \omega)Pr(\omega\vert b,a) \nonumber \\
	&=\sum_{\omega \in \Omega }Pr(\omega\vert b,a)[b' =b'_{a\omega}], \nonumber 
	\end{align}
	with Iversen bracket $[\cdot]$ and $b'_{a\omega}$ denoting the updated belief calculated using Equation~(\ref{bel-gussp}), after executing action $a$ and receiving observation $\omega$. The probability of receiving $\omega$ is:
	\begin{align}
	Pr(\omega\vert b,a) &=\! \sum_{x'\in X}Pr(\omega,x'\vert b,a) \nonumber \\
	&= \sum_{x'\in X} O(a,x',\omega)\sum_{x\in X} T(s,a,s')b(g'), \nonumber 
	\end{align}
	with $x=\langle s,g\rangle$ and $x'= \langle s',g'\rangle$. Since $\vert S \vert$ in the GUSSP is finite, a finite set of reachable beliefs in the GUSSP results in a finite set of reachable states in the belief MDP. This is a tree of depth $h$ with internal nodes for decisions and transitions, the branching factor is $\mathcal{O}(\vert \Omega\vert)$ for each horizon, $h$~\cite{papadimitriou1987complexity}. Therefore, the total number of reachable beliefs in the GUSSP is $\mathcal{O}(\vert \Omega\vert^h)$, and thus the resulting belief MDP may have $\mathcal{O}(\vert \Omega\vert^h)$ distinct reachable states.  
\end{proof}
In the worst case, the observation function may be unconstrained and all the beliefs may be unique. Since there is no discounting in a GUSSP and the horizon is unknown a priori, GUSSPs may have \emph{infinitely} many beliefs and their complexity class may be undecidable in the worst case~\cite{madani1999undecidability}. Hence, solving GUSSPs with non-myopic observations optimally is computationally intractable. 

We now prove that a myopic observation function results in a finite number of reachable beliefs in a GUSSP. 
\begin{proposition}
	A GUSSP with myopic observation function has a finite number of reachable beliefs. \label{finite-belief}
\end{proposition}
\begin{proof}
	By definition, a myopic observation function produces either belief-collapsing observations or no information at all. For each case, we first calculate the updated belief for the goal configurations using Equation~(\ref{g-update}). Therefore, $\forall x' \in X$ with $x'= \langle s',g'\rangle$: 
	\begin{equation}
	b'(g')\!=\!\frac{O(a,x',\omega)\,b(g)}{\sum_{x'} O(a,x',\omega)b(g)}. \nonumber
	\end{equation}
	\underline{Case 1: Belief-collapsing observation.} Trivially, when $O(a,x',\omega)\!=\!0 $, the updated belief is $b'(g')\!=\!0$. When $O(a,x',\omega)\!=\!1 $, the updated belief is $b'(g')\!=\!1$.
	\underline{Case 2: No information.} When the observation provides no information, $\forall a \in A, O(a,x',\omega)\!=\!1/\vert\Omega \vert$. Then, 
	\begin{equation}
	b'(g')= \frac{b(g)/\vert\Omega \vert}{\sum_{x'} b(g)/\vert\Omega \vert}=b(g). \nonumber
	\end{equation}
	Thus, $\forall g\!\in\!G$, a myopic observation function produces collapsed belief or retains the same belief, resulting in a finite number of reachable beliefs for a goal configuration. Since $\vert S\vert$ is finite, the belief update following Equation~(\ref{bel-gussp}) would result in finite number of reachable beliefs for a GUSSP.
\end{proof}
Hence, a myopic observation function weakly monotonically collapses beliefs, allowing us to simplify the problem further. We now show that a GUSSP reduces to an SSP, similar to the mapping from a POMDP to belief-MDP~\cite{kaelbling1998planning}.
\begin{proposition}
	A GUSSP reduces to an SSP.\label{prop-ssp}
\end{proposition} 
\begin{proof}
	We map the GUSSP to a belief MDP $\langle B, \mathcal{A}, \mathcal{\tau}, \mathcal{\rho} \rangle$ with a horizon $h$~\cite{kaelbling1998planning}, as in Proposition~\ref{prop:exp}. By Proposition~\ref{finite-belief}, a GUSSP with myopic observation function has a finite number of reachable beliefs and therefore, finite states in the belief-MDP. By construction, this belief-MDP is an SSP with the start state $\bar{s}_0\!=\!b_0$ and the goal states, $\bar{S}_G$, are the set of states with $\bar{b}(x)\!=\!1$ such that $\bar{b}(g)\!=\!1$ and  $s\!\in\!g$. Since there exists a proper policy in a GUSSP, the policy in this SSP is proper by construction. Thus, a GUSSP with myopic observation function reduces to an SSP.
\end{proof}
The reduction to an SSP facilitates solving GUSSPs using the existing rich suite of SSP algorithms. For ease of reference and clarity, we refer to the above-mentioned SSP as compiled-SSP in the rest of this paper.

\vspace{4pt}
\noindent{\textbf{Relation to Goal-POMDPs~}}
The Goal-POMDP~\cite{bonet2009solving} models a class of goal-based and shortest-path POMDPs with positive action costs and no discounting. The set of target (or goal) states, $\bar{P}$ have unique belief-collapsing observations.
Hence, a Goal-POMDP is a GUSSP when the partial observability is restricted to goals, the observations set is $2^{\bar{P}}\setminus\!\{\emptyset\}$, and observation function is myopic. 
\begin{proposition}
	GUSSP $\subset$ Goal-POMDP.
\end{proposition}

The observations in a Goal-POMDP are not constrained and may result in infinitely many reachable beliefs (Proposition~\ref{prop:exp}). This makes it computationally challenging to compute optimal policies~\cite{papadimitriou1987complexity}, unlike GUSSPs which are more tractable can be solved optimally (Proposition~\ref{prop-ssp}).

\vspace{4pt}
\noindent{\textbf{GUSSP with Deterministic Transitions~}}
A GUSSP with deterministic transitions presents an opportunity for further reduction in complexity. 
\begin{proposition}
	The optimal policy for a GUSSP with myopic observations and deterministic transitions is the minimum arborescence of a weighted and directed graph $Z$.
\end{proposition}
\begin{proof} 
	Consider a GUSSP with deterministic transitions and a dummy start state, $r$, that transitions to the actual start state with probability 1 and zero cost. This can be represented as a directed and weighted graph, $Z\!=\!(V,E,w)$, such that $V\!=\{r\} \cup \{x \in X \vert x\!=\!\langle s,g\rangle \land s\!\in\!P_G\}$; that is, the start state and the potential goals are the vertices. Each edge $e\!\in\!E$ denotes a trajectory in the GUSSP between vertices. The proper policy in a GUSSP ensures that there is at least one edge between each pair of vertices. The weight of an edge connecting $x,y\!\in\!V$ is $w(e)\!=\!d(x,y)(1\!-\!b(y))$, with $d(x,y)$ denoting the cost of the trajectory and $b(y)$ is the belief over $y$ being a goal. The minimum arborescence (directed minimum spanning tree) of this graph, $A$, contains trajectories such that the total weight is minimized, $\min_{A \in \mathcal{A}} w(A)$ with $w(A)\!=\!\sum_{e\in A}w(e) $. By construction, this gives the optimal order of visiting the potential goals and hence the optimal policy for the GUSSP with $V^*(s_0)\!=\!w(A)$.
\end{proof}

\section{Solving Compiled-SSPs}
We propose (i) an admissible heuristic for SSP solvers that accounts for the goal uncertainty and (ii) a determinization-based approach for solving the compiled-SSP.

\subsection{Admissible Heuristic}
In heuristic search-based SSP solvers, the heuristic function helps avoid visiting states that are provably irrelevant. An efficient heuristic for solving the compiled-SSP guides the search by 
accounting for the goal uncertainty. We propose a simple heuristic for the compiled-SSP that accounts for goal uncertainty and is calculated as follows: 
\[ h_{pg}(x) \triangleq \min_{g\in G} \Big((1-b(g))\,\min_{i\in g} d(x,i)\Big)\]
where $d(x,i)$ denotes the cost of the shortest trajectory to the potential goal $i$ from state $x$ and $b(g)$ is the agent's belief of $g$ being a true goal. Multiplying by the probability of a state not being a goal ($1-b(g)$) breaks ties in favor of configurations with a higher probability of being a goal, with a lower heuristic value. The following proposition shows that the proposed heuristic is admissible.
\begin{proposition}
	$h_{pg}$ is an admissible heuristic.
\end{proposition} 
\begin{proof}
	To show that $h_{pg}$ is admissible, we first show that $\min_{i\in g}d (x,i)$ is an admissible estimate of the expected cost of reaching a goal configuration $g$ from state $x$. Let $d^*(x,g)$ be the expected cost of reaching $g$ from $x$. Since $d (x,g)$ is the cost of the shortest trajectory to $g$ from $x$, $d (x,g)\!\leq\!d^*(x,g)$. If all paths exist from $x$ to all potential goal states $i\!\in\!g$, then by definition, the shortest trajectory to a goal configuration is the minimum distance to a potential goal in $g$. That is, $d (x,g)\!=\!\min_{i\in g}d (x,i)$ and therefore $\min_{i\in g} d (x,i)\!\leq\!d^*(x,g)$. Multiplying this value by the belief and using the minimum value over all possible goal configurations guarantees that $h_{pg}$ is an admissible estimate of the expected cost reaching a true goal configuration.
\end{proof}

\begin{table*}[t]  
	\centering \small
	\renewcommand{\arraystretch}{1.08}
	\caption{Comparisons of average cost and planning time (seconds) with bold titles indicating our techniques.}
	\resizebox{7.0in}{!} {
		\begin{tabular}{|c|r | r|r|r|r|r|r|r|r|r|r|}
			\hline
			& \multicolumn{2}{c|}{LAO* (Optimal solver)}
			& \multicolumn{2}{c|}{Flares(1)-$h_{min}$}  
			& \multicolumn{2}{c|}{\textbf{Flares(1)-$\mathbf{h_{pg}}$}}
			& \multicolumn{2}{c|}{\textbf{Det-MLG}}
			& \multicolumn{2}{c|}{\textbf{Det-CG}}
			\\ \hline
			Problem Instance  &  \multicolumn{1}{|p{1.3cm}|}{\centering Cost} &  \multicolumn{1}{|p{1cm}|}{\centering Time}  &  \multicolumn{1}{|p{1.3cm}|}{\centering Cost } &  \multicolumn{1}{|p{1cm}|}{\centering Time}  &  \multicolumn{1}{|p{1.3cm}|}{\centering Cost} &  \multicolumn{1}{|p{1cm}|}{\centering Time}  &  \multicolumn{1}{|p{1.3cm}|}{\centering Cost} &  \multicolumn{1}{|p{1cm}|}{\centering Time} & \multicolumn{1}{|p{1.3cm}|}{\centering Cost} &  \multicolumn{1}{|p{1cm}|}{\centering Time}
			\\ \hline \hline
			rover (20,6) & 28.25 & 14.99 & 
			35.35 $\pm$ 2.67 & 1.08 & 
			30.34 $\pm$ 2.37 & 0.17 & 
			36.71 $\pm$ 2.62 & 0.07 & 
			45.51 $\pm$ 3.22 & 0.06 \\
			rover (20,7) & 42.16 & 30.19 &  
			43.49 $\pm$ 1.62 & 1.17 & 
			45.07 $\pm$ 1.77 & 0.83 & 
			49.69 $\pm$ 1.91 & 0.02 & 
			48.36 $\pm$ 1.43 & 0.03 \\
			rover (30,8) & 36.96 & 190.92 & 
			38.21 $\pm$ 1.83 & 2.27 & 
			41.31 $\pm$ 1.97 & 0.16 & 
			38.54 $\pm$ 1.54 & 0.02 & 
			40.34 $\pm$ 1.82 & 0.03 \\
			rover (30,9) & 34.72 & 832.56 & 
			38.21 $\pm$ 2.54 & 7.56 & 
			43.32 $\pm$ 2.54 & 1.73 & 
			50.27 $\pm$ 2.58 & 0.88 & 
			49.49 $\pm$ 1.97  & 0.45 \\
			search (20,4) & 87.63 & 15.78 & 
			94.32 $\pm$ 0.58 & 1.45 & 
			93.32 $\pm$ 0.58 & 0.98 & 
			91.22 $\pm$ 0.67 & 1.05 & 
			90.42 $\pm$ 0.61 & 0.86 \\
			search (20,5) & 74.61 &  14.42 & 
			83.83 $\pm$ 0.56 & 2.99 & 
			81.91 $\pm$ 0.56 & 1.93 & 
			78.32 $\pm$ 0.56 & 1.98 & 
			79.74 $\pm$ 6.37 & 0.98 \\
			search (20,5) & 86.72 & 63.71 & 
			94.21 $\pm$ 0.79 & 6.21 & 
			91.18 $\pm$ 1.46 & 1.93 & 
			87.74 $\pm$ 0.65 & 0.66 & 
			89.98 $\pm$ 0.59 & 1.68 \\
			search (30,6) & 90.89 & 267.35 & 
			94.21 $\pm$ 1.35 & 117.63 & 
			103.77 $\pm$ 3.42 & 21.07 & 
			101.67 $\pm$ 1.61 & 12.68 & 
			92.94 $\pm$ 0.68 & 19.50 \\
			ev (-,5) & 2.34 & 8.16 & 
			3.29 $\pm$ 1.55 & 2.21 & 
			4.89 $\pm$ 1.36 & 0.92  & 
			5.15 $\pm$ 1.46 & 0.52 &
			7.17  $\pm$ 1.43 & 0.62 \\
			ev (-,6) & 3.46 & 10.79 & 
			4.89 $\pm$ 1.96 & 2.25 & 
			5.96 $\pm$ 1.96 & 1.14 & 
			7.15 $\pm$ 2.46 &  0.88 & 
			8.17  $\pm$ 1.43 & 0.79 \\
			\hline
	\end{tabular}}
		\vspace{-4pt}
	\label{tab:results}
\end{table*}

\subsection{Determinization}
Determinization is a popular approach for solving large SSPs as it simplifies the problem by replacing the probabilistic outcomes of an action with a single deterministic outcome~\cite{yoon2007ff,saisubramanian2019reduced}. We extend determinization to a GUSSP by ignoring the uncertainty about the goals. The agent plans to reach one potential goal (determinized goal) at a time, simplifying the problem to a smaller SSP. During execution, if the determinized goal is not a true goal, the agent replans for another unvisited potential goal. This approximation scheme offers considerable speedup over solving the compiled-SSP.  

We consider two determinization approaches: (i) most-likely goal determinization (DET-MLG) and (ii) closest-goal determinization (DET-CG). In the DET-MLG, the most-likely goal is determinized, based on its current belief. In DET-CG, the agent determinizes the closest goal based on the heuristic distance to the potential goal (with non-zero belief) from its current state. We resolve ties randomly.

\section{Experiments}
We begin with a comparison of different approximate solution techniques for solving the compiled-SSP on three domains in simulation. We then test the model on a real robot with three different initial belief settings.

\subsection{Evaluation in Simulation}
We experiment with three domains to evaluate the solution techniques in handling (i) location-based goal uncertainty (planetary rover domain, search and rescue domain) and (ii) temporal goal uncertainty (electric vehicle (EV) charging problem using real-world data). The expected cost of reaching the goal and run time (in seconds) are used as evaluation metrics. A uniform initial belief is considered for all the domains in these experiments. We solve the compiled-SSPs optimally using LAO*~\cite{hansen2001lao}, which is an optimal solver based on A*~\cite{hart1968formal} for solving MDPs with loops, and approximately using FLARES, a domain-independent state-of-the-art algorithm for solving large SSPs using horizon=1~\cite{pineda2017fast}, as well as the two determinization methods. 
The $h_{min}$ heuristic, computed using a labeled version of LRTA*~\cite{bonet2003labeled}, is used as a baseline for evaluating $h_{pg}$.


\begin{figure*}
	\setlength{\fboxsep}{0pt}
	\setlength{\fboxrule}{1pt}
	\centering
	\fbox{\includegraphics[width=1.7in, height=1.3in]{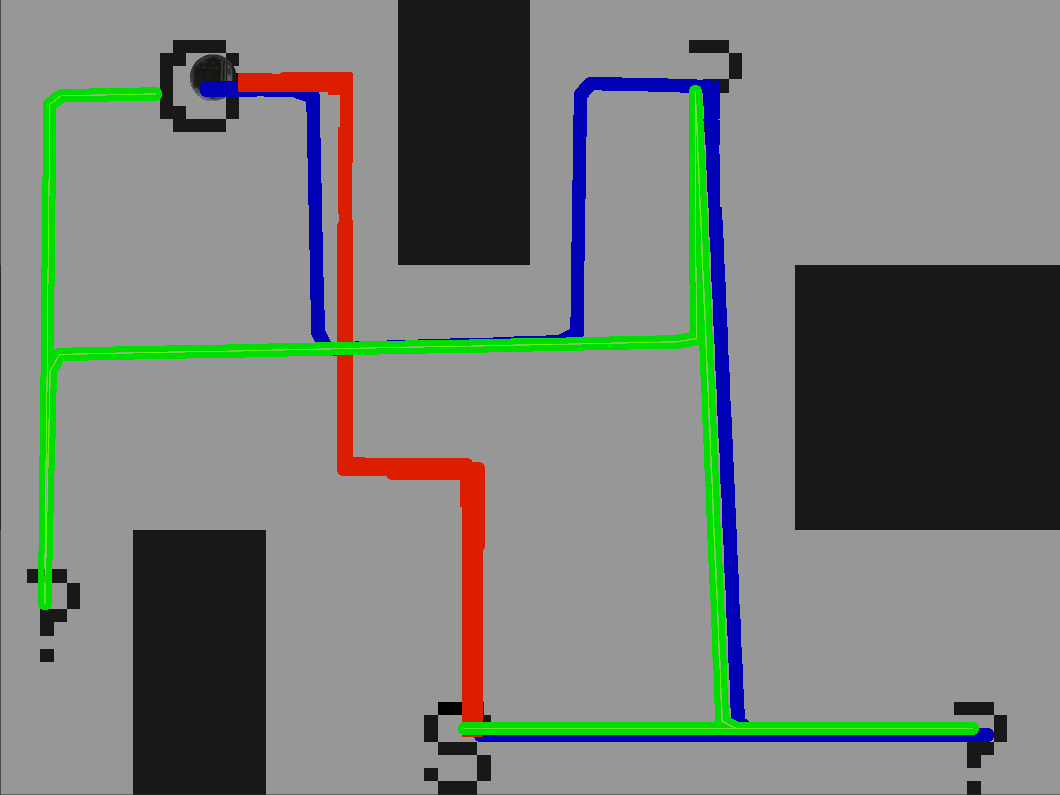}} 
	\fbox{\includegraphics[width=1.68in, height=1.3in]{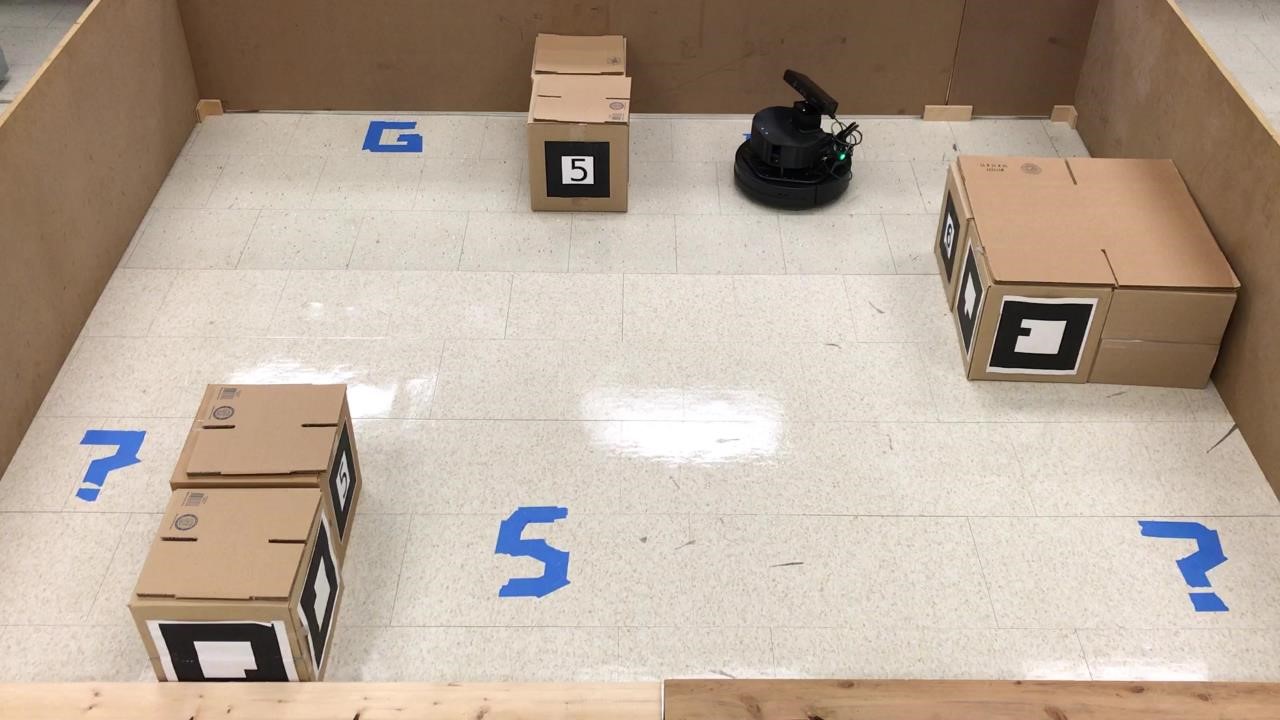}} 
	\fbox{\includegraphics[width=1.68in, height=1.3in]{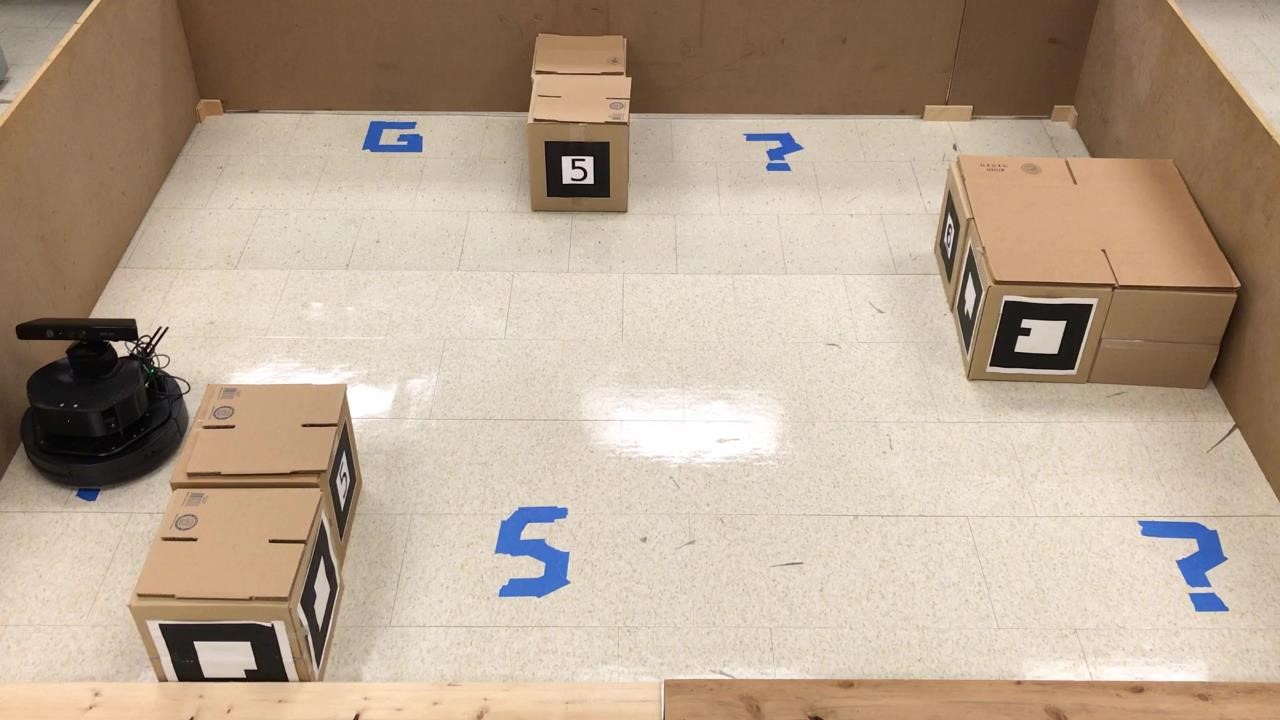}} 
	\fbox{\includegraphics[width=1.68in, height=1.3in]{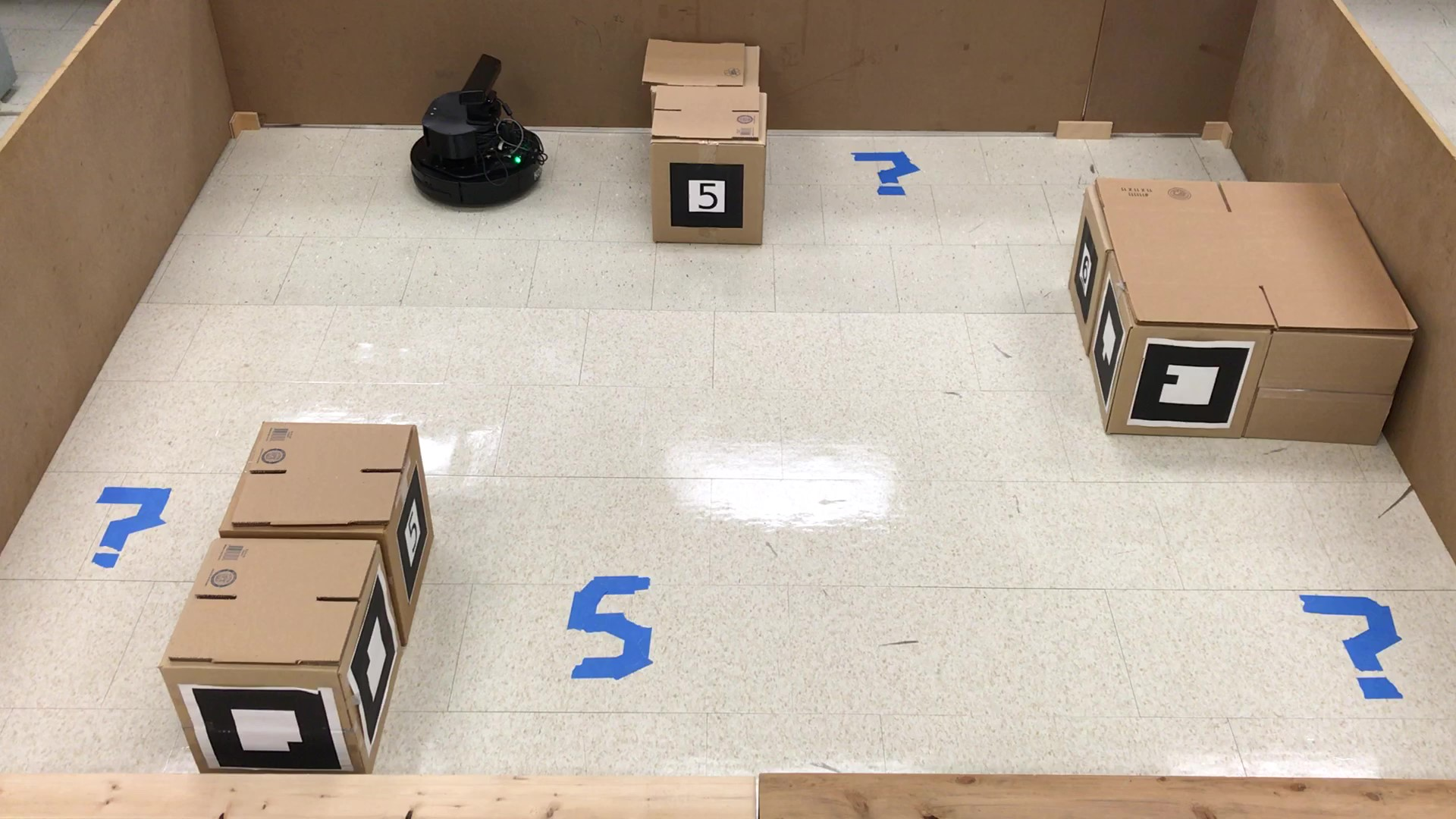}}
	\vspace{-8pt}
	\caption{Demonstration of the path taken by the robot with three different initial beliefs for the map in Figure~\ref{fig:sr}. The start state and the true goal state are denoted by S and G, respectively. The other potential goals are denoted by the question mark symbol. Green, blue, and red show the path taken by the robot with 0.1, 0.25, and 0.9 as the initial belief for the true goal state and equal probability for other potential goal states.}
	\label{fig:simulation}
\end{figure*}

\vspace{4pt}
\noindent{\textbf{Planetary Rover~}}
This domain models the rover science exploration~\cite{ZWBMlnai02,ong2010planning} that explores an environment described by a known map to collect a mineral sample. There are $n$ sample locations and the samples at each of these locations may be `good' or `bad', $\vert P_G \vert = n$. The rover knows its own position ($x,y$) exactly, as well as those of the samples but does not know which samples are `good'. The process terminates upon collecting a `good' sample. The actions include moving in all four directions, which succeed with a probability of $0.8$, and a \emph{sample} action which is deterministic. The \emph{sample} action costs $+2$ if the mineral is good and $+10$ otherwise; all other actions cost $+1$. 

\vspace{4pt}
\noindent{\textbf{Search and Rescue~}}
In this domain, an autonomous robot explores an environment described by a known map to find victims~\cite{pineda2015continual}. We modify the problem such that there are $m$ victims locations and $n$ total victims. The agent is aware of the potential victim locations and each location may or may not have victims. The exact locations of the victims are unknown to the robot a priori. The objective is to minimize the expected cost of saving all victims. The state factors include the robot's current location and a counter to indicate the number of victims saved so far. The observations indicate the presence of victims in each state. The actions include moving in all four directions and a SAVE action that saves all the victims in a state. The move actions cost $+1$ and are stochastic, succeeding with $0.8$ probability. The SAVE action is deterministic and costs $+2$. Any obstacle the robot may encounter is accounted for in the transition function.

\vspace{4pt}
\noindent{\textbf{Electric Vehicle Charging~}}
We experimented with the electric vehicle (EV) charging domain, operating in a vehicle-to-grid setting~\cite{saisubramanianoptimizing,saisubramanian2019reduced}, where the EV can charge and discharge energy from a smart grid. The objective is to devise a robust policy that is consistent with the owner's preferences, while minimizing the operational cost of the vehicle. We modified the problem such that parking duration of the EV is uncertain with $H$ denoting the horizon. The potential goals in this problem are the possible departure times. The EV can fully observe its current charge level and the time step. In our experiments, $\vert P_G\vert\!=\!n$ denotes that  $P_G\!=\!\{H,H-1,..,H\!-\!n\}$. Each $t$ is equivalent to 30 minutes in real time. If the EV's exit charge level does not meet the owner's desired exit charge level, a penalty is incurred. 

The battery capacity and the charge speeds for the EV are based on Nissan Leaf configuration and the action costs and peak hours are based on real data~\cite{pricing}. The charge levels and entry time data are based on charging schedules of electric cars over a four month duration in 2017 from a university campus. The data is clustered based on the entry and exit charges, and we selected 25 representative problem instances across clusters for our experiments.

\vspace{4pt}
\noindent{\textbf{Discussion~}}
Table~\ref{tab:results} shows the results of the five techniques on various problem instances, in terms of cost and runtime(s) respectively.
The grid size and the number of potential goals for each problem is indicated in parenthesis in the table. The results are averaged over 100 trials and standard errors are reported for the expected cost. The results for the EV domain are averaged over 25 problem instances. We experiment with no landmark states to demonstrate the performance in the worst case setting. 
In terms of expected costs, the performance of the approximate techniques are comparable. The runtimes for solving the problems optimally, however, scales rapidly as the number of potential goals increases. The advantage of using FLARES with $h_{pg}$ and the determinization techniques are more evident in the runtime savings. FLARES using our heuristic $h_{pg}$ is significantly faster than using the baseline $h_{min}$ heuristic. Both the determinization techniques are faster than solving the problem using FLARES. 

\subsection{Evaluation on a mobile robot} The robot experiment aims to visually explain how the belief distribution alters the robot's trajectory. 
Figure~\ref{fig:simulation} shows the results in a ROS simulation and on a real robot for a simple search and rescue problem with one agent and four potential victim locations for the map shown in Figure~\ref{fig:sr}. We test with three different initial beliefs: uniform, optimistic, and pessimistic. The corresponding belief of the true goal, $G$, in each belief setting is: $0.25$, $0.9$, and $0.1$, with the other potential goals having equal probability.

\section{Conclusion and Future Work}
The goal uncertain SSP (GUSSP) provides a natural model for representing real-world problems where it is non-trivial to identify the exact goals ahead of plan execution. While a general GUSSP could be intractable, we identify several tractable classes of GUSSPs and propose effective approaches for solving them. Specifically, we show that a GUSSP with a myopic observation function can be reduced to an SSP, allowing us to efficiently solve it using existing SSP solvers. We also propose an admissible heuristic that accounts for goal uncertainty in its estimation and a fast solver based on extending the notion of determinization to handle goal uncertainty. The results show that GUSSPs can be solved efficiently using scalable algorithms that do not rely on POMDP solvers. In the future, we aim to explore other conditions under which GUSSPs have a bounded set of beliefs that supports the development of efficient solvers, and examine the implications of goal uncertainty in multi-agent settings~\cite{AZaamas09} and other contexts.

\addtolength{\textheight}{-12cm}   


\bibliographystyle{IEEEtran}
\bibliography{References_POSSP-2}

\begin{thebibliography}{10}
\providecommand{\url}[1]{#1}
\csname url@rmstyle\endcsname
\providecommand{\newblock}{\relax}
\providecommand{\bibinfo}[2]{#2}
\providecommand\BIBentrySTDinterwordspacing{\spaceskip=0pt\relax}
\providecommand\BIBentryALTinterwordstretchfactor{4}
\providecommand\BIBentryALTinterwordspacing{\spaceskip=\fontdimen2\font plus
\BIBentryALTinterwordstretchfactor\fontdimen3\font minus
  \fontdimen4\font\relax}
\providecommand\BIBforeignlanguage[2]{{%
\expandafter\ifx\csname l@#1\endcsname\relax
\typeout{** WARNING: IEEEtran.bst: No hyphenation pattern has been}%
\typeout{** loaded for the language `#1'. Using the pattern for}%
\typeout{** the default language instead.}%
\else
\language=\csname l@#1\endcsname
\fi
#2}}

\bibitem{bertsekas1991analysis}
D.~P. Bertsekas and J.~N. Tsitsiklis, ``An analysis of stochastic shortest path
  problems,'' \emph{Mathematics of Operations Research,}, vol.~16, pp.
  580--595, 1991.

\bibitem{kitano1999robocup}
H.~Kitano, S.~Tadokoro, I.~Noda, H.~Matsubara, T.~Takahashi, A.~Shinjou, and
  S.~Shimada, ``{RoboCup} rescue: Search and rescue in large-scale disasters as
  a domain for autonomous agents research,'' in \emph{IEEE Conference on
  Systems, Man, and Cybernetics}, 1999.

\bibitem{pineda2015continual}
L.~Pineda, T.~Takahashi, H.-T. Jung, S.~Zilberstein, and R.~Grupen, ``Continual
  planning for search and rescue robots,'' in \emph{IEEE Conference on Humanoid
  Robots}, 2015.

\bibitem{hansen2007indefinite}
E.~A. Hansen, ``Indefinite-horizon {POMDP}s with action-based termination,'' in
  \emph{AAAI}, 2007.

\bibitem{stone2016search}
L.~D. Stone, J.~O. Royset, and A.~R. Washburn, ``Search for a stationary
  target,'' in \emph{Optimal Search for Moving Targets}.\hskip 1em plus 0.5em
  minus 0.4em\relax Springer, 2016, pp. 9--48.

\bibitem{lanillos2012minimum}
P.~Lanillos, E.~Besada-Portas, G.~Pajares, and J.~J. Ruz, ``Minimum time search
  for lost targets using cross entropy optimization,'' in \emph{IROS}, 2012.

\bibitem{trevizan13finding}
F.~Trevizan and M.~Veloso, ``{F}inding objects through stochastic shortest path
  problems,'' in \emph{AAMAS}, 2013.

\bibitem{nie2016searching}
X.~Nie, L.~L. Wong, and L.~P. Kaelbling, ``Searching for physical objects in
  partially known environments,'' in \emph{ICRA}, 2016.

\bibitem{ong2010planning}
S.~Ong, S.~W. Png, D.~Hsu, and W.~S. Lee, ``Planning under uncertainty for
  robotic tasks with mixed observability,'' \emph{International Journal of
  Robotics Research}, vol.~29, pp. 1053--1068, 2010.

\bibitem{kaelbling1998planning}
L.~P. Kaelbling, M.~L. Littman, and A.~R. Cassandra, ``Planning and acting in
  partially observable stochastic domains,'' \emph{Artificial Intelligence},
  vol. 101, pp. 99--134, 1998.

\bibitem{papadimitriou1987complexity}
C.~H. Papadimitriou and J.~N. Tsitsiklis, ``The complexity of markov decision
  processes,'' \emph{Mathematics of Operations Research}, vol.~12, pp.
  441--450, 1987.

\bibitem{patek2001partially}
S.~D. Patek, ``On partially observed stochastic shortest path problems,'' in
  \emph{IEEE Conference on Decision and Control}, 2001.

\bibitem{bonet2009solving}
B.~Bonet and H.~Geffner, ``Solving {POMDP}s: {RTDP}-bel vs. point-based
  algorithms.'' in \emph{IJCAI}, 2009.

\bibitem{littman1997probabilistic}
M.~L. Littman, ``Probabilistic propositional planning: Representations and
  complexity,'' in \emph{AAAI}, 1997.

\bibitem{biswas2013multi}
J.~Biswas and M.~Veloso, ``Multi-sensor mobile robot localization for diverse
  environments,'' in \emph{Robot Soccer World Cup}.\hskip 1em plus 0.5em minus
  0.4em\relax Springer, 2013, pp. 468--479.

\bibitem{madani1999undecidability}
O.~Madani, S.~Hanks, and A.~Condon, ``On the undecidability of probabilistic
  planning and infinite-horizon partially observable markov decision
  problems,'' in \emph{AAAI}, 1999.

\bibitem{yoon2007ff}
S.~Yoon, A.~Fern, and R.~Givan, ``{FF-Replan}: A baseline for probabilistic
  planning,'' in \emph{ICAPS}, 2007.

\bibitem{saisubramanian2019reduced}
S.~Saisubramanian and S.~Zilberstein, ``Adaptive outcome selection for planning
  with reduced models,'' in \emph{IROS}, 2019.

\bibitem{hansen2001lao}
E.~A. Hansen and S.~Zilberstein, ``{LAO}*: A heuristic search algorithm that
  finds solutions with loops,'' \emph{Artificial Intelligence}, vol. 129, pp.
  35--62, 2001.

\bibitem{hart1968formal}
P.~E. Hart, N.~J. Nilsson, and B.~Raphael, ``A formal basis for the heuristic
  determination of minimum cost paths,'' \emph{IEEE Transactions on Systems
  Science and Cybernetics}, vol.~4, pp. 100--107, 1968.

\bibitem{pineda2017fast}
L.~Pineda, K.~Wray, and S.~Zilberstein, ``Fast {SSP} solvers using
  short-sighted labeling,'' in \emph{AAAI}, 2017.

\bibitem{bonet2003labeled}
B.~Bonet and H.~Geffner, ``Labeled {RTDP}: Improving the convergence of
  real-time dynamic programming,'' in \emph{ICAPS}, 2003.

\bibitem{ZWBMlnai02}
S.~Zilberstein, R.~Washington, D.~S. Bernstein, and A.-I. Mouaddib,
  ``Decision-theoretic control of planetary rovers,'' in \emph{Revised Papers
  from the Intl. Seminar on Advances in Plan-Based Control of Robotic
  Agents}.\hskip 1em plus 0.5em minus 0.4em\relax Springer-Verlag, 2002, pp.
  270--289.

\bibitem{saisubramanianoptimizing}
S.~Saisubramanian, S.~Zilberstein, and P.~Shenoy, ``Optimizing electric vehicle
  charging through determinization,'' in \emph{Scheduling and Planning
  Applications Workshop (SPARK), ICAPS}, 2017.

\bibitem{pricing}
Eversource, ``Time of use rates,'' https://www.eversource.com/clp/vpp/vpp.aspx,
  2017.

\bibitem{AZaamas09}
C.~Amato and S.~Zilberstein, ``Achieving goals in decentralized {POMDP}s,'' in
  \emph{AAMAS}, 2009.

\end{thebibliography}

\end{document}